\documentclass[11pt]{article}

\usepackage[a4paper,margin=24mm]{geometry}
\usepackage[T1]{fontenc}
\usepackage[utf8]{inputenc}
\usepackage{amsmath,amssymb,amsthm,mathtools,bm}
\usepackage{newtxtext,newtxmath}
\usepackage{microtype}
\usepackage{booktabs,array}
\usepackage{graphicx}
\usepackage{subcaption}
\usepackage{placeins}

\usepackage{enumitem}
\usepackage{tikz}
\usetikzlibrary{patterns,arrows.meta}
\usepackage{float}
\usetikzlibrary{arrows.meta,positioning,fit,calc,shapes.geometric,backgrounds}
\usepackage{pgfplots}
\pgfplotsset{compat=1.18}
\usepackage[numbers,sort&compress]{natbib}
\usepackage{authblk}
\usepackage{xcolor}
\usepackage[colorlinks=true,allcolors=blue!55!black]{hyperref}
\usepackage[nameinlink,noabbrev]{cleveref}

\setlist{nosep,leftmargin=1.5em}
\numberwithin{equation}{section}
\newtheorem{theorem}{Theorem}[section]

\newtheorem{remark}[theorem]{Remark}

\providecommand{\SHOWCHANGES}{0}
\ifnum\SHOWCHANGES=1
  \newcommand{\rev}[1]{{\color{red}#1}}
  \newenvironment{revision}{\color{red}}{}
\else
  \newcommand{\rev}[1]{#1}
  
\fi

\newcommand{\R}{\mathbb{R}}
\newcommand{\N}{\mathcal{N}}
\newcommand{\Sigmag}{\boldsymbol{\Sigma}_{\boldsymbol{\gamma}}}
\newcommand{\Sigmal}{\boldsymbol{\Sigma}_{\boldsymbol{\lambda}}}
\newcommand{\Sigmay}{\boldsymbol{\Sigma}_{y}}
\newcommand{\Sigmapost}{\boldsymbol{\Sigma}_{\mathrm{post}}}

\newcommand{\Rdata}{\mathbf{R}_{\mathrm{data}}}
\newcommand{\Bgamma}{\mathbf{B}_{\boldsymbol{\gamma}}}
\newcommand{\Agl}{\mathbf{A}_{\boldsymbol{\gamma},\boldsymbol{\lambda}}}
\newcommand{\I}{\mathbf{I}}
\newcommand{\diag}{\operatorname{diag}}
\newcommand{\Null}{\operatorname{Null}}
\newcommand{\Range}{\operatorname{Range}}
\newcommand{\rank}{\operatorname{rank}}
\newcommand{\Cov}{\operatorname{Cov}}
\newcommand{\Var}{\operatorname{Var}}

\title{\textbf{Exact Spectral Uncertainty Decomposition for Linear-Gaussian Bayesian M/EEG Inverse Problems}}

\author[1,2]{Ismail Huseynov}
\author[1,2,3]{Stefan Haufe}
\affil[1]{Physikalisch-Technische Bundesanstalt (PTB), Berlin, Germany}
\affil[2]{Technische Universit\"at Berlin, Berlin, Germany}
\affil[3]{Charit\'e--Universit\"atsmedizin Berlin, Berlin, Germany}
\affil[ ]{\texttt{ismail.huseynov@ptb.de} \quad \texttt{haufe@tu-berlin.de}}
\date{}

\begin{document}
\maketitle
\vspace{-1.6em}

\begin{abstract}
Distributed magneto- and electroencephalographic (M/EEG) source reconstruction is a severely underdetermined linear inverse problem. Bayesian source imaging regularizes the inversion through source-prior and sensor-noise covariance models, but the posterior covariance does not directly distinguish uncertainty that is constrained by the measurements from uncertainty that is unobservable under the sensing geometry. We study this distinction for linear-Gaussian M/EEG models with diagonal source-prior and sensor-noise covariances, covering both Type-I inference and Type-II inference conditional on learned covariance hyperparameters.

After restricting the analysis to the prior-supported source space, standardizing the prior, and whitening the sensor noise, the complete singular spectrum of the resulting forward operator separates \textit{likelihood-visible} and \textit{likelihood-null} directions. This yields an \textit{additive uncertainty decomposition} of the conditional posterior covariance: uncertainty in \textit{likelihood-visible} directions is reduced according to their spectral information content, whereas \textit{likelihood-null} uncertainty remains unchanged from the prior. The same spectral representation yields a complementary\textit{ prior-minus-reduction identity}, showing that covariance reduction is confined to the \textit{likelihood-visible} subspace and is quantified by \textit{spectral markers} describing likelihood information and prior-variance retention and reduction.

Fixed-orientation MEG simulations with adaptive noise learning illustrate the framework. Dense Bayesian minimum-norm inference retains a large likelihood-null component, whereas sparse $\gamma$-MAP strongly reduces the fitted prior support. Within that support, the spectral markers distinguish strongly informed directions from a weakly constrained direction. Thus, prior-support selection determines which source coordinates remain in the fitted model, while the spectral decomposition quantifies how strongly directions within that source space are constrained by the measurements.
\end{abstract}


\section{Introduction}
\label{sec:intro}

Electroencephalography (EEG) and magnetoencephalography (MEG) measure electric potentials and magnetic fields generated by neural currents. Under the quasi-static approximation to Maxwell's equations, discretization of the electromagnetic forward problem gives
\begin{equation}
  \mathbf Y
  =
  \mathbf L\mathbf X+\mathbf E,
  \label{eq:forward-matrix}
\end{equation}
where
$\mathbf L\in\R^{M\times N}$ is the lead-field matrix, $\mathbf X\in\R^{N\times T}$ contains the amplitudes of $N$ candidate sources at $T$ time instants, $\mathbf Y\in\R^{M\times T}$ contains the corresponding measurements at $M$ sensors and $T$ time points, and $\mathbf E\in\R^{M\times T}$ denotes additive measurement noise
\citep{Sarvas1987,Hamalainen1993,Mosher1999,Baillet2001}. In distributed source imaging, typically $M\ll N$, so the inverse problem is severely underdetermined. The lead field is non-injective
and often strongly ill-conditioned: exact null directions lead to non-uniqueness, whereas directions associated with small nonzero
singular values are only weakly constrained by the measurements and are sensitive to measurement and modeling errors \citep{Kaipio2005,Stuart2010}.

Bayesian M/EEG inversion regularizes this problem by combining the likelihood with a prior distribution on the sources. 
We formulate the inverse problem within a linear-Gaussian Bayesian model with diagonal source-prior and sensor-noise covariances.

Once these covariance parameters are fixed, the conditional source posterior is Gaussian. In Type-I Bayesian learning, the covariance parameters are prescribed independently of the current data, which includes classical Gaussian-prior and minimum-norm formulations. In Type-II Bayesian learning, the covariance hyperparameters are instead estimated from the marginal likelihood and subsequently inserted into the same conditional Gaussian posterior \citep{WipfNagarajan2009,SekiharaNagarajan2015,HashemiEtAl2021}. Such evidence-based covariance learning is widely used in Bayesian electromagnetic source imaging and allows the effective source and noise covariance structure to be adapted to the observed data
\citep{Friston2008,WipfEtAl2010,HashemiEtAl2021,CaiEtAl2021}.

Posterior covariance provides a natural measure of conditional source uncertainty, but in its conventional matrix form it does not directly reveal which source directions are strongly constrained by the measurements and which remain weakly or not constrained at all. In spectral terms, directions associated with large singular values are strongly informed by the data, whereas directions associated with small nonzero singular values are only weakly informed and directions in the nullspace are exactly unobservable. For practical source reconstruction, the distinction between weakly informed and exactly unobservable directions may be less consequential, because both can carry little usable measurement information. Moreover, the effective separation between these cases can be influenced by measurement noise and forward-model uncertainty, even when the forward operator is treated as fixed in the analysis. The spectral formulation developed below therefore characterizes measurement informativeness continuously across source directions while retaining the exact nullspace as the mathematically unobservable limiting case.


Closely related spectral structure has been extensively studied in large-scale and high-dimensional Bayesian inverse problems, especially
for PDE-constrained models, where direct manipulation of posterior covariance matrices is computationally prohibitive
\citep{Flath2011,BuiThanh2013}. The prior-preconditioned data-misfit Hessian and related generalized eigenvalue problems provide a natural
description of directions that are strongly or weakly informed by the data. This structure has been used to construct scalable low-rank
posterior covariance approximations
\citep{Flath2011,BuiThanh2013,Spantini2015}, to compress space-time posterior covariances and accelerate hierarchical sampling
\citep{BennerQiuStoll2018,BrownSaibabaVallelian2018}, and to define likelihood-informed subspaces for nonlinear inverse problems \citep{CuiEtAl2014}. The central objective of this literature is
typically computational: dominant data-informed directions are retained so that the remaining high-dimensional structure need not be treated at full cost.

Our objective is different. Within the conditional linear-Gaussian model, with the lead field and covariance parameters treated as given, we retain the complete singular structure of the noise-whitened, prior-standardized forward operator to obtain an \emph{exact} spectral decomposition of the posterior covariance. The construction applies to both nonsingular and singular source-prior covariances: in the nonsingular case, prior standardization is defined by the principal square root of the prior covariance, whereas in the singular case it is restricted to the \textit{prior-supported source space} without requiring an inverse or pseudoinverse.

Prior-preconditioned Hessian methods have primarily exploited the dominant data-informed spectrum to construct low-rank approximations of posterior covariance and related computational reductions
\citep{Flath2011,BuiThanh2013,Spantini2015}. In contrast, we retain the complete positive spectrum together with the \emph{likelihood-null} subspace. This yields an \emph{exact} decomposition of the conditional posterior covariance into \emph{likelihood-visible} and \emph{likelihood-null} contributions, with the latter remaining unchanged from its prior value. The same spectral representation also yields a complementary prior-minus-reduction identity, showing that covariance reduction induced by conditioning acts exclusively through the \emph{likelihood-visible} subspace.


This separation also provides a direct interpretation in physical source coordinates. The prior-supported  likelihood-null component corresponds to source perturbations that are both supported by the prior and invisible to the lead field, i.e., directions in $\Range(\Sigmag)\cap\Null(\mathbf L)$. In physical source space, the likelihood-null directions are precisely those source perturbations that remain supported by the prior but produce no sensor response. This separates two different mechanisms: source directions that are unobservable because of the lead-field geometry and source directions that are excluded because the prior assigns them zero variance. The complete spectrum further quantifies the amount of likelihood information and the corresponding retention and reduction of prior variance along the likelihood-visible directions.


The prior-minus-posterior covariance structure itself is not new. Related posterior-update identities are standard in linear-Gaussian
Bayesian inverse problems and underlie several prior-preconditioned Hessian methods \citep{Flath2011,BuiThanh2013,Spantini2015}. The emphasis here is instead on the complete visible/null accounting: the covariance reduction is localized explicitly to the likelihood-visible part of the prior-supported source space, while the exact likelihood-null component is retained. A related image/kernel viewpoint has also appeared in linearized Laplace approximations for Bayesian neural networks \citep{Roy2024}, although there the decomposition concerns local parameter-function geometry in an approximate nonlinear model rather than the exact conditional covariance of a linear-Gaussian inverse problem.

In the numerical study, we illustrate the decomposition in a controlled fixed-orientation M/EEG simulation with adaptive noise learning. Dense BMN with sLORETA normalization retains the full source space and, consequently, a large exact likelihood-null component, whereas sparse $\gamma$-MAP restricts the fitted prior support to a small set containing all simulated generators. Within this reduced support, the spectral markers further distinguish strongly informed directions from an additional retained component that remains only weakly constrained by the measurements. The experiment therefore highlights an important distinction between \emph{prior-support selection} and \emph{likelihood informativeness}: Type-II sparsification can reduce the effective source space without implying that every retained direction is equally supported by the data. 

The main contributions of this work are as follows. First, we derive an exact spectral factorization of the conditional Gaussian posterior on the prior-supported source space and identify its likelihood-visible and likelihood-null components. Second, we obtain an exact
source-space posterior covariance decomposition and a complementary prior-minus-reduction identity, showing that data-driven covariance
reduction acts exclusively through the mapped likelihood-visible subspace. Third, we characterize the physical unobservable source
subspace as the intersection of the prior-supported source space with the lead-field nullspace, thereby separating sensing-induced
non-identifiability from prior-support restriction. Finally, we derive spectral information and variance markers that quantify, direction by direction, the strength of likelihood information, the fraction of prior variance remaining in the posterior, and the fraction removed by the data. Together, these results provide an exact conditional framework for determining which parts of M/EEG source uncertainty are reduced by the available measurements and which remain unresolved by the linear forward model.

\section{Bayesian M/EEG source model}
\label{sec:model}

\subsection{Linear forward model}

For each $t=1,\ldots,T$, consider the linear-Gaussian M/EEG model
\begin{equation}
  \mathbf y(t)
  =
  \mathbf L\mathbf x(t)
  +
  \boldsymbol\varepsilon(t),
  \qquad
  \mathbf x(t)\mid\boldsymbol\gamma
  \sim
  \N(\mathbf0,\Sigmag),
  \qquad
  \boldsymbol\varepsilon(t)\mid\boldsymbol\lambda
  \sim
  \N(\mathbf0,\Sigmal),
  \label{eq:model}
\end{equation}
with
\begin{equation*}
  \Sigmag
  :=
  \diag(\gamma_1,\ldots,\gamma_N)
  \succeq0,
  \qquad
  \Sigmal
  :=
  \diag(\lambda_1,\ldots,\lambda_M)
  \succ0.
\end{equation*}
The vector $\boldsymbol\gamma\in\R_+^N$ parameterizes the source-space prior covariance and controls the prior variance assigned to each candidate source. The vector $\boldsymbol\lambda\in\R_{>0}^M$ parameterizes the sensor-space noise covariance and allows different
noise variances across sensors.

We assume that, conditional on the covariance parameters
$\boldsymbol\gamma$ and $\boldsymbol\lambda$, the source vectors $\mathbf x(t)$ and noise vectors $\boldsymbol\varepsilon(t)$ are independent across time, with the same source and noise covariance
parameters shared over $t=1,\ldots,T$. After marginalizing the latent source vector, each sensor vector has the Gaussian distribution
\begin{equation}
  \mathbf y(t)\mid\boldsymbol\gamma,\boldsymbol\lambda
  \sim
  \N\!\left(
    \mathbf0,
    \Sigmay(\boldsymbol\gamma,\boldsymbol\lambda)
  \right),
\end{equation}
with marginal sensor-space covariance
\begin{equation}
  \Sigmay(\boldsymbol\gamma,\boldsymbol\lambda)
  =
  \underbrace{
  \mathbf L\Sigmag\mathbf L^{\mathsf T}
  }_{\text{epistemic}}
  +
  \underbrace{
  \Sigmal.
  }_{\text{aleatoric}}
  \label{eq:sensor-covariance}
\end{equation}

Thus, the marginal covariance of the measurements admits a natural \textit{additive uncertainty decomposition} in sensor space. The first
term is the source uncertainty propagated through the lead field and represents the epistemic contribution conditional on the chosen source-prior parameters. The second term represents aleatoric variability generated by measurement noise.

\subsection{Type-I and Type-II Bayesian learning}

The distinction between Type-I and Type-II Bayesian learning concerns how the covariance hyperparameters are chosen, not the form of the conditional source posterior. Once $(\boldsymbol\gamma,\boldsymbol\lambda)$ are fixed, both approaches use the same linear-Gaussian posterior distribution.

For fixed covariance parameters, source estimation can be formulated as maximum-a-posteriori (MAP) inference with respect to the latent source vector $\mathbf x(t)$. In the present linear-Gaussian model, the posterior is Gaussian and its MAP estimate coincides with its posterior mean \citep{SekiharaNagarajan2015}. Hence, the conditional posterior takes the form
\begin{equation}
  p\!\left(
  \mathbf x(t)\mid
  \mathbf y(t),\boldsymbol\gamma,\boldsymbol\lambda
  \right)
  =
  \N\!\left(
  \boldsymbol\mu_{\mathrm{post}}(t),
  \Sigmapost
  \right),
  \label{eq:conditional-posterior}
\end{equation}
with posterior summaries
\begin{equation}
  \boldsymbol\mu_{\mathrm{post}}(t)
  =
  \Sigmag\mathbf L^{\mathsf T}
  \Sigmay(\boldsymbol\gamma,\boldsymbol\lambda)^{-1}
  \mathbf y(t),
  \qquad
  \Sigmapost
  =
  \Sigmag
  -
  \Sigmag\mathbf L^{\mathsf T}
  \Sigmay(\boldsymbol\gamma,\boldsymbol\lambda)^{-1}
  \mathbf L\Sigmag .
  \label{eq:posterior-summary}
\end{equation}


\paragraph{Type-I Bayesian learning.}
In Type-I learning, source inference is performed for a specified prior and noise model, with the associated covariance or regularization parameters fixed a priori or selected using auxiliary or data-dependent criteria. These parameters need not be independent of the data. In practice, regularization parameters may be
selected using criteria such as the L-curve, the discrepancy principle, or spatial cross-validation
\citep{HansenOLeary1993,EnglHankeNeubauer1996,NegiEtAl2026}.
Noise levels may likewise be specified from prior knowledge or estimated from baseline measurements, and residual-based criteria may be used to select the regularization strength by matching the residual level to the presumed noise level \citep{HamalainenIlmoniemi1994,SekiharaNagarajan2015}.
For homoscedastic Gaussian prior and noise models, the regularization parameter is determined by the relative scaling of the prior and noise variances, so that estimating or fixing these quantities also determines the corresponding regularization strength. Once these parameters are specified, source estimation proceeds directly from the conditional Gaussian posterior in
\cref{eq:conditional-posterior,eq:posterior-summary}
\citep{HashemiEtAl2021,SekiharaNagarajan2015}.

\paragraph{Type-II Bayesian learning.}
In Type-II Bayesian learning, the covariance hyperparameters are estimated from the data by maximizing the marginal likelihood, also referred to as the model evidence. Marginalizing the latent source vectors yields, up to an additive constant independent of the hyperparameters,
\begin{equation}
  \mathcal L_{\mathrm{II}}
  (\boldsymbol\gamma,\boldsymbol\lambda)
  =
  \frac{T}{2}
  \log\det
  \Sigmay(\boldsymbol\gamma,\boldsymbol\lambda)
  +
  \frac12
  \sum_{t=1}^{T}
  \mathbf y(t)^{\mathsf T}
  \Sigmay(\boldsymbol\gamma,\boldsymbol\lambda)^{-1}
  \mathbf y(t).
  \label{eq:typeII-objective}
\end{equation}

The covariance hyperparameters are obtained by minimizing
\eqref{eq:typeII-objective} over the admissible parameter domain. The resulting point estimates $(\widehat{\boldsymbol\gamma},
\widehat{\boldsymbol\lambda})$ are then fixed and inserted into the conditional Gaussian posterior in \cref{eq:posterior-summary}. Consequently, the posterior covariance considered here quantifies source uncertainty conditional on these plug-in hyperparameter estimates and does not propagate uncertainty in the hyperparameters themselves.

For M/EEG source imaging, the sparse Type-II formulation considered here corresponds to the SBL or $\gamma$-MAP framework underlying Champagne-type algorithms
\citep{WipfNagarajan2009,WipfEtAl2010,SekiharaNagarajan2015}. In the standard $\gamma$-MAP formulation, the source-variance hyperparameters are learned while the noise level is fixed or estimated from baseline measurements. In the adaptive-noise extension, the source and noise hyperparameters are jointly learned within the same Type-II framework, typically through alternating updates, leading to $\gamma$-MAP with adaptive noise learning \citep{HashemiEtAl2021,CaiEtAl2021}.

\subsection{Prior standardization and noise whitening}

Throughout this work, the source-prior covariance is assumed to be diagonal,
\[
\Sigmag=\diag(\boldsymbol\gamma), \qquad \gamma_i\geq 0,
\]
so that the prior variances are associated directly with the source
coordinates. This parametrization covers both the nonsingular case, in which all source variances are strictly positive, and the singular case, in which one or more source variances vanish. The latter arises naturally in sparse Type-II/SBL estimation, where evidence optimization may drive selected variance parameters to zero. A singular prior could also encode the a priori exclusion of specific source locations, although such locations would commonly be removed from the forward model itself.

To treat the nonsingular and singular cases within the same construction, we define the index set of source coordinates carrying nonzero prior variance,
\begin{equation*}
  \mathcal I_{\boldsymbol\gamma}
  :=
  \left\{
    i\in\{1,\ldots,N\}:\gamma_i>0
  \right\},
  \qquad
  q:=|\mathcal I_{\boldsymbol\gamma}|.
\end{equation*}
These coordinates span the \emph{prior-supported source space}. Let
$\mathbf P_{\boldsymbol\gamma}\in\R^{N\times q}$ denote the canonical selection matrix whose columns are the standard basis vectors indexed by $\mathcal I_{\boldsymbol\gamma}$, and define
\begin{equation}
  \Bgamma
  :=
  \mathbf P_{\boldsymbol\gamma}
  \diag\!\left(
    \sqrt{\gamma_i}:
    i\in\mathcal I_{\boldsymbol\gamma}
  \right)\in\R^{N\times q},
  \qquad
  \Sigmag
  =
  \Bgamma\Bgamma^{\mathsf T}.
  \label{eq:active-factor}
\end{equation}
When $\Sigmag$ is singular, this construction provides a reduced
parametrization of the prior-supported source space without requiring either $\Sigmag^{-1}$ or a pseudoinverse.

\begin{remark}
If $\gamma_i>0$ for all $i=1,\ldots,N$, then
$\mathcal I_{\boldsymbol\gamma}=\{1,\ldots,N\}$, $q=N$, and
$\Bgamma=\Sigmag^{1/2}$, so that the construction reduces to the usual
principal square-root standardization.

The diagonal parametrization adopted here represents prior variability
directly in the source coordinates. More general positive-semidefinite source priors can analogously be represented through a factor spanning their positive-eigenvalue subspace, which permits prior structure or sparsity to be expressed in a noncanonical basis. Such formulations arise, for example, in multiple-sparse-prior approaches to M/EEG source reconstruction \citep{FristonEtAl2008} and in inverse problems employing transform-domain sparsity, as commonly used in MRI reconstruction \citep{LustigDonohoPauly2007}. These extensions are not required for the source-wise diagonal parametrization considered in the present analysis.
\end{remark}

The source vector can then be represented on the \textit{prior-supported source space} as
\begin{equation}
  \mathbf x(t)
  =
  \Bgamma\mathbf z(t),
  \qquad
  \mathbf z(t)
  \sim
  \N(\mathbf0,\I_q).
  \label{eq:latent-map}
\end{equation}
Thus, $\mathbf z(t)$ provides standardized latent coordinates in which the source prior is isotropic. 


We next whiten the measurement model with respect to the sensor-noise
covariance. We assume that the sensor-noise covariance $\Sigmal$ is full rank, equivalently positive definite, so that $\Sigmal^{-1/2}$ is well defined. Define
\begin{equation}
  \widetilde{\mathbf y}(t)
  :=
  \Sigmal^{-1/2}\mathbf y(t),
  \qquad
  \widetilde{\boldsymbol\varepsilon}(t)
  :=
  \Sigmal^{-1/2}\boldsymbol\varepsilon(t),
  \label{eq:whitened-variables}
\end{equation}
and the noise-whitened, prior-standardized forward operator
\begin{equation}
  \Agl
  :=
  \Sigmal^{-1/2}\mathbf L\Bgamma
  \in\R^{M\times q}.
  \label{eq:whitened-operator}
\end{equation}
Substitution of $\mathbf x(t)=\Bgamma\mathbf z(t)$ into the original forward model gives
\begin{equation}
  \widetilde{\mathbf y}(t)
  =
  \Agl\mathbf z(t)
  +
  \widetilde{\boldsymbol\varepsilon}(t),
  \qquad
  \widetilde{\boldsymbol\varepsilon}(t)
  \sim
  \N(\mathbf0,\I_M).
  \label{eq:whitened-model}
\end{equation}
Hence, both the latent source prior and the sensor-noise covariance are represented by identity matrices, while $\Agl$ contains the remaining forward-model geometry relative to the prior and noise scales.

\begin{remark}
The whitening construction above assumes that the sensor-noise covariance has full rank in the measurement space. This assumption may be violated when the sensor data have undergone rank-reducing linear preprocessing, as in common-average referencing or the removal of selected components following independent component analysis (ICA) \citep{HyvarinenOja2000}. In such cases, the measurement model should first be restricted to the effective sensor subspace associated with the nonzero spectrum of the transformed noise covariance.
Whitening is then performed with respect to the positive-definite covariance induced on that reduced subspace. This provides the sensor-space analogue of the reduced prior parametrization introduced above for singular $\Sigmag$.
\end{remark}

The posterior in the standardized coordinates follows directly from Bayes' rule:
\begin{align*}
  p\!\left(
    \mathbf z(t)\mid\widetilde{\mathbf y}(t)
  \right)
  &\propto
  p\!\left(
    \widetilde{\mathbf y}(t)\mid\mathbf z(t)
  \right)
  p\!\left(\mathbf z(t)\right)
  \nonumber\\
  &\propto
  \exp\!\left\{
    -\frac12
    \left\|
      \widetilde{\mathbf y}(t)-\Agl\mathbf z(t)
    \right\|_2^2
    -
    \frac12
    \left\|
      \mathbf z(t)
    \right\|_2^2
  \right\}.
\end{align*}
Collecting the terms that depend on $\mathbf z(t)$ gives
\begin{equation*}
  -\frac12
  \mathbf z(t)^{\mathsf T}
  \left(
    \I_q+\Agl^{\mathsf T}\Agl
  \right)
  \mathbf z(t)
  +
  \mathbf z(t)^{\mathsf T}
  \Agl^{\mathsf T}
  \widetilde{\mathbf y}(t).
\end{equation*}
Completing the square therefore yields
\begin{equation*}
  \mathbf z(t)\mid\widetilde{\mathbf y}(t)
  \sim
  \N\!\left(
    \boldsymbol\mu_{z\mid y}(t),
    \boldsymbol\Sigma_{z\mid y}
  \right),
\end{equation*}
where
\begin{equation*}
  \boldsymbol\Sigma_{z\mid y}
  =
  \left(
    \I_q+\Agl^{\mathsf T}\Agl
  \right)^{-1},
  \qquad
  \boldsymbol\mu_{z\mid y}(t)
  =
  \boldsymbol\Sigma_{z\mid y}
  \Agl^{\mathsf T}
  \widetilde{\mathbf y}(t).
  \label{eq:latent-posterior}
\end{equation*}
Mapping this Gaussian posterior back through $\mathbf x(t)=\Bgamma\mathbf z(t)$ gives
\begin{equation*}
  \boldsymbol\mu_{\mathrm{post}}(t)
  =
  \Bgamma\boldsymbol\mu_{z\mid y}(t),
  \qquad
  \Sigmapost
  =
  \Bgamma
  \boldsymbol\Sigma_{z\mid y}
  \Bgamma^{\mathsf T},
  \label{eq:source-posterior-latent}
\end{equation*}
which recovers the source-space posterior summaries in \cref{eq:posterior-summary}.


For fixed covariance parameters $\boldsymbol\gamma$ and
$\boldsymbol\lambda$, consider the negative log-likelihood as a
function of the standardized source coefficients $\mathbf z$:
\[
  \Phi_t(\mathbf z)
  :=
  \frac12
  \left\|
    \widetilde{\mathbf y}(t)-\Agl\mathbf z
  \right\|_2^2,
\]
where terms independent of $\mathbf z$ have been omitted.
Its Hessian with respect to $\mathbf z$ is
\begin{equation}
  \mathbf H_{\boldsymbol\gamma,\boldsymbol\lambda}
  :=
  \nabla_{\mathbf z}^2\Phi_t(\mathbf z)
  =
  \Agl^{\mathsf T}\Agl
  =
  \Bgamma^{\mathsf T}
  \mathbf L^{\mathsf T}
  \Sigmal^{-1}
  \mathbf L
  \Bgamma
  \succeq0.
  \label{eq:pphessian}
\end{equation}
Thus, $\mathbf H_{\boldsymbol\gamma,\boldsymbol\lambda}$ is the
\emph{prior-preconditioned data-misfit Hessian}: it expresses the
likelihood curvature in prior-standardized source coordinates.
The covariance hyperparameters are held fixed when taking these
derivatives. Since the forward model is linear, this Hessian is
constant with respect to $\mathbf z$ and independent of the observed
measurement values.

For a unit eigenvector $\mathbf v_i$ with eigenvalue $\kappa_i$,
\[
  \kappa_i
  =
  \mathbf v_i^{\mathsf T}
  \mathbf H_{\boldsymbol\gamma,\boldsymbol\lambda}\mathbf v_i
  =
  \|\Agl\mathbf v_i\|_2^2.
\]
Hence, $\kappa_i$ measures the likelihood precision along that
standardized direction relative to the unit prior precision.
The Gaussian prior contributes $\mathbf I_q$ to the negative
log-posterior Hessian, giving posterior covariance
$(\mathbf I_q+\mathbf H_{\boldsymbol\gamma,\boldsymbol\lambda})^{-1}$.
Large eigenvalues therefore indicate strongly informed directions,
whereas zero eigenvalues identify directions along which the
likelihood is constant and the standardized prior variance is
preserved.

\begin{remark}
The nonzero eigenvalues of
$\mathbf H_{\boldsymbol\gamma,\boldsymbol\lambda}$ are the squared nonzero singular values of $\Agl$. Prior-preconditioned Hessian methods commonly exploit spectral decay to construct low-rank posterior covariance approximations
\citep{Flath2011,BuiThanh2013,Spantini2015}. Related constructions have been developed for covariance compression and hierarchical sampling \citep{BennerQiuStoll2018,BrownSaibabaVallelian2018},while likelihood-informed subspaces use a related prior-versus-likelihood geometry for nonlinear inverse problems \citep{CuiEtAl2014}. Here, we instead retain the complete positive spectrum together with the exact zero eigenspace because our objective is exact conditional covariance accounting rather than low-rank approximation.
\end{remark}

\section{Exact spectral posterior decomposition}
\label{sec:decomposition}

\subsection{Likelihood-visible and likelihood-null coordinates}

Once the covariance parameters $\boldsymbol\gamma$ and $\boldsymbol\lambda$ are fixed, the noise-whitened, prior-standardized operator $\Agl\in\R^{M\times q}$ is common to all
$t=1,\ldots,T$. Hence, the same spectral subspaces apply at every time point. 

Let
\[
  r:=\rank(\Agl)\leq\min\{M,q\}.
\]
The compact singular value decomposition of $\Agl$ is
\begin{equation}
  \Agl
  =
  \mathbf U_r
  \mathbf S_r
  \mathbf V_r^{\mathsf T},
  \qquad
  \mathbf S_r
  =
  \diag(s_1,\ldots,s_r),
  \qquad
  s_1\geq\cdots\geq s_r>0,
  \label{eq:svd}
\end{equation}
where
$\mathbf U_r\in\R^{M\times r}$,
$\mathbf S_r\in\R^{r\times r}$, and
$\mathbf V_r\in\R^{q\times r}$, with
\[
  \mathbf U_r^{\mathsf T}\mathbf U_r=\I_r,
  \qquad
  \mathbf V_r^{\mathsf T}\mathbf V_r=\I_r.
\]
The columns of $\mathbf U_r$ form an orthonormal basis of $\Range(\Agl)\subseteq\R^M$, whereas the columns of $\mathbf V_r$ form an orthonormal basis of
\[
  \Range(\Agl^{\mathsf T})
  =
  \Null(\Agl)^\perp
  \subseteq\R^q.
\]
We refer to the latter as the \emph{likelihood-visible subspace}. The scalars $s_i$ are the positive singular values of $\Agl$, and $s_i^2$ are the corresponding positive eigenvalues of
$\mathbf H_{\boldsymbol\gamma,\boldsymbol\lambda}$.

To represent the complementary directions in the prior-supported source space, let
$\mathbf V_0\in\R^{q\times(q-r)}$ be an orthonormal basis of $\Null(\Agl)$, which we call the \emph{likelihood-null subspace}, and define
\[
  \mathbf V
  :=
  [\,\mathbf V_r\ \mathbf V_0\,]
  \in\R^{q\times q}.
\]
Because
\[
  \R^q
  =
  \Range(\Agl^{\mathsf T})
  \oplus
  \Null(\Agl),
\]
the matrix $\mathbf V$ is orthogonal. Hence,
\[
  \mathbf V_r^{\mathsf T}\mathbf V_r=\I_r,
  \qquad
  \mathbf V_0^{\mathsf T}\mathbf V_0=\I_{q-r},
  \qquad
  \mathbf V_r^{\mathsf T}\mathbf V_0=\mathbf0,
\]
and
\[
  \mathbf V_r\mathbf V_r^{\mathsf T}
  +
  \mathbf V_0\mathbf V_0^{\mathsf T}
  =
  \I_q,
  \qquad
  \Agl\mathbf V_0=\mathbf0.
\]
The null-space basis $\mathbf V_0$ will be used to express
the posterior uncertainty in directions that are not constrained by the likelihood. A full SVD would additionally introduce an orthonormal basis of $\Null(\Agl^{\mathsf T})$ in sensor space. These left-null directions do not enter $\mathbf H_{\boldsymbol\gamma,\boldsymbol\lambda}$ and are therefore not needed for the source-space posterior decomposition. For a dense prior, $q=N$, whereas sparse Type-II learning may yield $q<N$.

For each $t=1,\ldots,T$, define the coordinates of $\mathbf z(t)$ in the two orthogonal subspaces by
\begin{equation}
  \boldsymbol\alpha(t)
  :=
  \mathbf V_r^{\mathsf T}\mathbf z(t)
  \in\R^r,
  \qquad
  \boldsymbol\beta(t)
  :=
  \mathbf V_0^{\mathsf T}\mathbf z(t)
  \in\R^{q-r},
  \qquad
  \mathbf z(t)
  =
  \mathbf V_r\boldsymbol\alpha(t)
  +
  \mathbf V_0\boldsymbol\beta(t).
  \label{eq:alpha-beta}
\end{equation}
Here, $\boldsymbol\alpha(t)$ contains the \textit{likelihood-visible coordinates} and $\boldsymbol\beta(t)$ the \textit{likelihood-null coordinates}. The reconstruction of $\mathbf z(t)$ follows from the orthogonal decomposition of $\R^q$ above. Since $\mathbf z(t)\sim\N(\mathbf0,\I_q)$ and $\mathbf V$ is orthogonal,
\[
  \boldsymbol\alpha(t)\sim\N(\mathbf0,\I_r),
  \qquad
  \boldsymbol\beta(t)\sim\N(\mathbf0,\I_{q-r}),
\]
and the two coordinate blocks are independent under the prior.

Substituting \cref{eq:alpha-beta} into the whitened model gives
\begin{equation}
  \widetilde{\mathbf y}(t)
  =
  \mathbf U_r\mathbf S_r\boldsymbol\alpha(t)
  +
  \widetilde{\boldsymbol\varepsilon}(t),
  \label{eq:alpha-beta-model}
\end{equation}
because $\Agl\mathbf V_r=\mathbf U_r\mathbf S_r$
and $\Agl\mathbf V_0=\mathbf0$. Hence, the likelihood depends only on the likelihood-visible
coordinates $\boldsymbol\alpha(t)$; the likelihood-null coordinates $\boldsymbol\beta(t)$ produce no change in the conditional mean of the whitened measurements.

\begin{theorem}[Posterior in likelihood-visible and likelihood-null coordinates]
\label{thm:joint-visible-null}

For every $t=1,\ldots,T$, define
\begin{equation}
  \mathbf C_{\alpha}
  :=
  \left(
    \I_r+\mathbf S_r^2
  \right)^{-1},
  \qquad
  \overline{\boldsymbol\alpha}(t)
  :=
  \mathbf C_{\alpha}
  \mathbf S_r
  \mathbf U_r^{\mathsf T}
  \widetilde{\mathbf y}(t).
  \label{eq:alpha-moments}
\end{equation}
Then
\begin{equation}
  \begin{bmatrix}
    \boldsymbol\alpha(t)\\
    \boldsymbol\beta(t)
  \end{bmatrix}
  \Bigm|
  \widetilde{\mathbf y}(t)
  \sim
  \N\!\left(
  \begin{bmatrix}
    \overline{\boldsymbol\alpha}(t)\\
    \mathbf0
  \end{bmatrix},
  \begin{bmatrix}
    \mathbf C_{\alpha} & \mathbf0\\
    \mathbf0 & \I_{q-r}
  \end{bmatrix}
  \right).
  \label{eq:joint-posterior}
\end{equation}
Consequently,
$p(\boldsymbol\beta(t)\mid\widetilde{\mathbf y}(t))
=p(\boldsymbol\beta(t))$, and the likelihood-visible and likelihood-null coordinates remain independent in the conditional posterior.
\end{theorem}

\begin{proof}
By orthogonality of $\mathbf V$, the prior factorizes as
\[
  p(\boldsymbol\alpha(t),\boldsymbol\beta(t))
  =
  p(\boldsymbol\alpha(t))
  p(\boldsymbol\beta(t)),
\]
with
\[
  \boldsymbol\alpha(t)\sim\N(\mathbf0,\I_r),
  \qquad
  \boldsymbol\beta(t)\sim\N(\mathbf0,\I_{q-r}).
\]
From \cref{eq:alpha-beta-model}, the likelihood is independent of $\boldsymbol\beta(t)$. Hence, Bayes' rule gives
\begin{align*}
  p\!\left(
    \boldsymbol\alpha(t),
    \boldsymbol\beta(t)
    \mid
    \widetilde{\mathbf y}(t)
  \right)
  &\propto
  \exp\!\left[
    -\frac12
    \left\|
      \widetilde{\mathbf y}(t)
      -
      \mathbf U_r\mathbf S_r
      \boldsymbol\alpha(t)
    \right\|_2^2
    -
    \frac12
    \|\boldsymbol\alpha(t)\|_2^2
  \right]
  \\
  &\qquad\times
  \exp\!\left[
    -\frac12
    \|\boldsymbol\beta(t)\|_2^2
  \right].
\end{align*}
Using $\mathbf U_r^{\mathsf T}\mathbf U_r=\I_r$, the terms involving $\boldsymbol\alpha(t)$ reduce, up to an additive constant, to
\[
  -\frac12
  \boldsymbol\alpha(t)^{\mathsf T}
  \left(
    \I_r+\mathbf S_r^2
  \right)
  \boldsymbol\alpha(t)
  +
  \boldsymbol\alpha(t)^{\mathsf T}
  \mathbf S_r
  \mathbf U_r^{\mathsf T}
  \widetilde{\mathbf y}(t).
\]
Completing the square gives
\[
  \boldsymbol\alpha(t)
  \mid
  \widetilde{\mathbf y}(t)
  \sim
  \N\!\left(
    \overline{\boldsymbol\alpha}(t),
    \mathbf C_{\alpha}
  \right).
\]
The remaining factor
\[
  \exp\!\left[
    -\frac12
    \|\boldsymbol\beta(t)\|_2^2
  \right]
\]
is identical to the prior factor of $\boldsymbol\beta(t)$. Therefore,
\[
  \boldsymbol\beta(t)
  \mid
  \widetilde{\mathbf y}(t)
  \sim
  \N(\mathbf0,\I_{q-r}),
\]
and the posterior factorizes between the two coordinate blocks, yielding \cref{eq:joint-posterior}.
\end{proof}

\subsection{Source-space posterior covariance decomposition}

The likelihood-visible and likelihood-null decomposition derived in the standardized latent coordinates can now be mapped back to the physical
source space.

\begin{theorem}[Exact source-space posterior covariance decomposition]
\label{thm:additive}

For every $t=1,\ldots,T$, the conditional source posterior mean is
\begin{equation}
  \boldsymbol\mu_{\mathrm{post}}(t)
  =
  \Bgamma\mathbf V_r
  \overline{\boldsymbol\alpha}(t),
  \label{eq:spectral-posterior-mean}
\end{equation}
and the conditional posterior covariance admits the exact decomposition
\begin{equation}
  \Sigmapost
  =
  \boldsymbol\Sigma_{\mathrm{post}}^{\mathrm{vis}}
  +
  \boldsymbol\Sigma_{\mathrm{post}}^{\mathrm{null}},
  \label{eq:additive}
\end{equation}
where
\begin{align}
  \boldsymbol\Sigma_{\mathrm{post}}^{\mathrm{vis}}
  :=
  \Bgamma\mathbf V_r
  \left(\I_r+\mathbf S_r^2\right)^{-1}
  \mathbf V_r^{\mathsf T}
  \Bgamma^{\mathsf T}, \qquad
  \boldsymbol\Sigma_{\mathrm{post}}^{\mathrm{null}}
  :=
  \Bgamma\mathbf V_0
  \mathbf V_0^{\mathsf T}
  \Bgamma^{\mathsf T}.
  \label{eq:visible-null-block}
\end{align}
Both covariance components are positive semidefinite. Moreover, the likelihood-null posterior covariance is identical to the prior covariance restricted to $\Null(\Agl)$ and mapped to source space through $\Bgamma$.
\end{theorem}

\begin{proof}
From the compact SVD
\[
  \Agl
  =
  \mathbf U_r\mathbf S_r\mathbf V_r^{\mathsf T},
\]
we obtain
\[
  \Agl^{\mathsf T}\Agl
  =
  \mathbf V_r
  \mathbf S_r^2
  \mathbf V_r^{\mathsf T}.
\]
Using the orthogonal decomposition
\[
  \I_q
  =
  \mathbf V_r\mathbf V_r^{\mathsf T}
  +
  \mathbf V_0\mathbf V_0^{\mathsf T},
\]
it follows that
\begin{align*}
  \I_q+\Agl^{\mathsf T}\Agl
  &=
  \mathbf V_r
  \left(\I_r+\mathbf S_r^2\right)
  \mathbf V_r^{\mathsf T}
  +
  \mathbf V_0\mathbf V_0^{\mathsf T}
  \\
  &=
  \begin{bmatrix}
    \mathbf V_r & \mathbf V_0
  \end{bmatrix}
  \begin{bmatrix}
    \I_r+\mathbf S_r^2 & \mathbf0\\
    \mathbf0 & \I_{q-r}
  \end{bmatrix}
  \begin{bmatrix}
    \mathbf V_r & \mathbf V_0
  \end{bmatrix}^{\mathsf T}.
\end{align*}
Since
$\mathbf V=[\,\mathbf V_r\ \mathbf V_0\,]$ is orthogonal, inversion gives
\begin{align*}
  \boldsymbol\Sigma_{z\mid y}
  &=
  \left(
    \I_q+\Agl^{\mathsf T}\Agl
  \right)^{-1}
  \\
  &=
  \mathbf V_r
  \left(\I_r+\mathbf S_r^2\right)^{-1}
  \mathbf V_r^{\mathsf T}
  +
  \mathbf V_0
  \mathbf V_0^{\mathsf T}.
\end{align*}
For the latent posterior mean,
\[
  \boldsymbol\mu_{z\mid y}(t)
  =
  \boldsymbol\Sigma_{z\mid y}
  \Agl^{\mathsf T}
  \widetilde{\mathbf y}(t),
\]
and
\[
  \Agl^{\mathsf T}
  =
  \mathbf V_r
  \mathbf S_r
  \mathbf U_r^{\mathsf T}.
\]
Therefore,
\begin{align*}
  \boldsymbol\mu_{z\mid y}(t)
  &=
  \left[
    \mathbf V_r
    \left(\I_r+\mathbf S_r^2\right)^{-1}
    \mathbf V_r^{\mathsf T}
    +
    \mathbf V_0\mathbf V_0^{\mathsf T}
  \right]
  \mathbf V_r
  \mathbf S_r
  \mathbf U_r^{\mathsf T}
  \widetilde{\mathbf y}(t)
  \\
  &=
  \mathbf V_r
  \left(\I_r+\mathbf S_r^2\right)^{-1}
  \mathbf S_r
  \mathbf U_r^{\mathsf T}
  \widetilde{\mathbf y}(t)
  \\
  &=
  \mathbf V_r
  \overline{\boldsymbol\alpha}(t),
\end{align*}
where we used
$\mathbf V_r^{\mathsf T}\mathbf V_r=\I_r$ and
$\mathbf V_0^{\mathsf T}\mathbf V_r=\mathbf0$.

Mapping the latent posterior through
$\mathbf x(t)=\Bgamma\mathbf z(t)$ gives
\[
  \boldsymbol\mu_{\mathrm{post}}(t)
  =
  \Bgamma
  \boldsymbol\mu_{z\mid y}(t)
  =
  \Bgamma\mathbf V_r
  \overline{\boldsymbol\alpha}(t),
\]
and
\begin{align*}
  \Sigmapost
  &=
  \Bgamma
  \boldsymbol\Sigma_{z\mid y}
  \Bgamma^{\mathsf T}
  \\
  &=
  \Bgamma\mathbf V_r
  \left(\I_r+\mathbf S_r^2\right)^{-1}
  \mathbf V_r^{\mathsf T}
  \Bgamma^{\mathsf T}
  +
  \Bgamma\mathbf V_0
  \mathbf V_0^{\mathsf T}
  \Bgamma^{\mathsf T}.
\end{align*}
This proves \cref{eq:additive,eq:visible-null-block}.

Finally, for every $\mathbf a\in\R^N$,
\[
  \mathbf a^{\mathsf T}
  \boldsymbol\Sigma_{\mathrm{post}}^{\mathrm{vis}}
  \mathbf a
  =
  \left\|
    \left(\I_r+\mathbf S_r^2\right)^{-1/2}
    \mathbf V_r^{\mathsf T}
    \Bgamma^{\mathsf T}\mathbf a
  \right\|_2^2
  \geq0,
\]
and
\[
  \mathbf a^{\mathsf T}
  \boldsymbol\Sigma_{\mathrm{post}}^{\mathrm{null}}
  \mathbf a
  =
  \left\|
    \mathbf V_0^{\mathsf T}
    \Bgamma^{\mathsf T}\mathbf a
  \right\|_2^2
  \geq0.
\]
Hence both covariance components are positive semidefinite. It remains to verify that the likelihood-null covariance is unchanged
from the corresponding prior covariance. Since
$\mathbf z(t)\sim\N(\mathbf0,\I_q)$, the orthogonal projection of the latent prior onto $\Null(\Agl)=\Range(\mathbf V_0)$ is
\[
  \mathbf z_{\mathrm{null}}(t)
  :=
  \mathbf V_0\mathbf V_0^{\mathsf T}\mathbf z(t),
\]
with prior covariance
\begin{align*}
  \Cov\!\left(
    \mathbf z_{\mathrm{null}}(t)
  \right)
  &=
  \mathbf V_0\mathbf V_0^{\mathsf T}
  \I_q
  \mathbf V_0\mathbf V_0^{\mathsf T}
  \\
  &=
  \mathbf V_0\mathbf V_0^{\mathsf T},
\end{align*}
where we used
$\mathbf V_0^{\mathsf T}\mathbf V_0=\I_{q-r}$.

On the other hand, by
\cref{thm:joint-visible-null},
\[
  \boldsymbol\beta(t)
  \mid
  \widetilde{\mathbf y}(t)
  \sim
  \N(\mathbf0,\I_{q-r}),
\]
and therefore the posterior covariance of the latent likelihood-null component
\[
  \mathbf z_{\mathrm{null}}(t)
  =
  \mathbf V_0\boldsymbol\beta(t)
\]
is
\[
  \Cov\!\left(
    \mathbf z_{\mathrm{null}}(t)
    \mid
    \widetilde{\mathbf y}(t)
  \right)
  =
  \mathbf V_0\mathbf V_0^{\mathsf T}.
\]
Hence, the likelihood-null covariance is identical before and after conditioning.

Finally, mapping this component to physical source space through
$\mathbf x=\Bgamma\mathbf z$ gives
\[
  \Bgamma
  \mathbf V_0\mathbf V_0^{\mathsf T}
  \Bgamma^{\mathsf T}
  =
  \boldsymbol\Sigma_{\mathrm{post}}^{\mathrm{null}},
\]
which is therefore exactly the prior covariance restricted to the likelihood-null subspace and mapped to source space.
\end{proof}

\begin{remark}
The likelihood-visible and likelihood-null subspaces are orthogonal in the standardized latent space because
$\mathbf V_r^{\mathsf T}\mathbf V_0=\mathbf0$. Their source-space images $\Bgamma\mathbf V_r$ and $\Bgamma\mathbf V_0$ need not be Euclidean-orthogonal when the source prior is anisotropic. Nevertheless, the covariance decomposition remains exact because the visible and null latent posterior blocks have zero cross-covariance. The linear map $\Bgamma$ therefore preserves the additive positive-semidefinite covariance identity in \cref{eq:additive}. 
\end{remark}

\subsection{Likelihood-driven covariance reduction}

For linear-Gaussian Bayesian inverse problems, conditioning yields the standard
prior-minus-reduction representation
\citep{Flath2011,BuiThanh2013,Spantini2015},
\[
  \Sigmapost
  =
  \Sigmag-\Rdata,
  \qquad
  \Rdata
  :=
  \Sigmag\mathbf L^{\mathsf T}
  \boldsymbol\Sigma_y^{-1}
  \mathbf L\Sigmag .
\]
Here, $\Rdata$ represents the prior covariance removed by conditioning on the measurements. Our interest is in resolving this reduction spectrally within the prior-supported source space and determining how it is distributed between the \emph{likelihood-visible} and \emph{likelihood-null} components.

In the prior-standardized coordinates, define the covariance reduction by
\begin{equation}
  \mathbf R_z
  :=
  \I_q-\boldsymbol\Sigma_{z\mid y}.
  \label{eq:standardized-reduction-def}
\end{equation}

\begin{theorem}[Spectral localization of covariance reduction]
\label{thm:reduction}

The standardized covariance reduction in
\cref{eq:standardized-reduction-def} admits the spectral representation
\begin{equation}
  \mathbf R_z
  =
  \mathbf V_r
  \mathbf S_r^2
  \left(\I_r+\mathbf S_r^2\right)^{-1}
  \mathbf V_r^{\mathsf T}.
  \label{eq:standardized-reduction}
\end{equation}
Its corresponding representation in physical source coordinates is
\begin{equation}
  \Rdata
  :=
  \Bgamma\mathbf R_z\Bgamma^{\mathsf T}
  =
  \Bgamma\mathbf V_r
  \mathbf S_r^2
  \left(\I_r+\mathbf S_r^2\right)^{-1}
  \mathbf V_r^{\mathsf T}
  \Bgamma^{\mathsf T}.
  \label{eq:reduction}
\end{equation}

Define the prior covariance carried by the \emph{likelihood-visible} and \emph{likelihood-null} subspaces as
\[
  \boldsymbol\Sigma_{\mathrm{prior}}^{\mathrm{vis}}
  :=
  \Bgamma\mathbf V_r\mathbf V_r^{\mathsf T}\Bgamma^{\mathsf T},
  \qquad
  \boldsymbol\Sigma_{\mathrm{prior}}^{\mathrm{null}}
  :=
  \Bgamma\mathbf V_0\mathbf V_0^{\mathsf T}\Bgamma^{\mathsf T}.
\]
Together with the posterior decomposition established in
\cref{thm:additive}, these components satisfy
\begin{equation}
  \boldsymbol\Sigma_{\mathrm{prior}}^{\mathrm{vis}}
  =
  \boldsymbol\Sigma_{\mathrm{post}}^{\mathrm{vis}}
  +
  \Rdata,
  \qquad
  \boldsymbol\Sigma_{\mathrm{prior}}^{\mathrm{null}}
  =
  \boldsymbol\Sigma_{\mathrm{post}}^{\mathrm{null}}.
  \label{eq:block-accounting}
\end{equation}
Moreover,
\[
  \Rdata\succeq0,
  \qquad
  \Range(\Rdata)
  =
  \Range(\Bgamma\mathbf V_r).
\]
Thus, covariance reduction induced by conditioning is localized entirely to the mapped \emph{likelihood-visible} subspace, whereas the covariance carried by the \emph{likelihood-null} subspace is preserved.
\end{theorem}

\begin{proof}
Using the standardized posterior decomposition established in \cref{thm:additive}, the likelihood-null contribution cancels in \cref{eq:standardized-reduction-def}, giving
\[
  \mathbf R_z
  =
  \mathbf V_r
  \left[
    \I_r-
    \left(\I_r+\mathbf S_r^2\right)^{-1}
  \right]
  \mathbf V_r^{\mathsf T}.
\]
The elementary resolvent identity
\[
  \I_r-
  \left(\I_r+\mathbf S_r^2\right)^{-1}
  =
  \mathbf S_r^2
  \left(\I_r+\mathbf S_r^2\right)^{-1}
\]
then yields
\[
  \mathbf R_z
  =
  \mathbf V_r
  \mathbf S_r^2
  \left(\I_r+\mathbf S_r^2\right)^{-1}
  \mathbf V_r^{\mathsf T},
\]
which proves \cref{eq:standardized-reduction}.

Mapping the standardized reduction through $\Bgamma$ gives
\[
  \Rdata
  =
  \Bgamma\mathbf R_z\Bgamma^{\mathsf T},
\]
and substitution of \cref{eq:standardized-reduction} yields
\cref{eq:reduction}.

Combining \cref{eq:reduction} with the
\emph{likelihood-visible} posterior component from
\cref{thm:additive} gives
\[
  \boldsymbol\Sigma_{\mathrm{prior}}^{\mathrm{vis}}
  -
  \boldsymbol\Sigma_{\mathrm{post}}^{\mathrm{vis}}
  =
  \Rdata.
\]
The \emph{likelihood-null} posterior component in
\cref{thm:additive} coincides with its prior counterpart, yielding
\cref{eq:block-accounting}.

Finally,
\[
  \mathbf S_r^2
  \left(\I_r+\mathbf S_r^2\right)^{-1}
  \succ0,
\]
because $\mathbf S_r$ contains only strictly positive singular values. Hence $\Rdata\succeq0$. Since $\Bgamma\mathbf V_r$ has full column rank,
\[
  \Range(\Rdata)
  =
  \Range(\Bgamma\mathbf V_r).
\]
\end{proof}

\begin{remark}
The prior-preconditioned data-misfit Hessian and its dominant spectral directions have been used extensively to construct low-rank posterior covariance approximations and prior-to-posterior covariance updates in
linear-Gaussian Bayesian inverse problems
\citep{Flath2011,BuiThanh2013,Spantini2015}. The emphasis of these approaches is primarily computational, with the dominant data-informed directions retained to approximate the posterior covariance.

The present analysis instead retains the complete positive spectrum together with the exact \emph{likelihood-null} subspace. This permits a complete visible/null accounting of the covariance reduction in the prior-supported source space: the reduction is localized to the mapped \emph{likelihood-visible} subspace, while the covariance associated with the
\emph{likelihood-null} subspace is preserved from prior to posterior.
\end{remark}

\subsection{Prior support and likelihood visibility in physical source space}
\label{sec:physical-support-geometry}

The likelihood-null structure has so far been characterized in the
prior-standardized coordinates through $\Null(\Agl)$. Its interpretation in
physical source coordinates requires distinguishing two conceptually
independent geometric properties. The source-prior covariance determines
\emph{prior support} through $\Range(\Sigmag)$ and
$\Null(\Sigmag)$, whereas the lead field determines
\emph{likelihood visibility} through $\Range(\mathbf L^{\mathsf T})$ and
$\Null(\mathbf L)$. The former specifies which source directions are assigned
prior variance, while the latter specifies which source directions can produce
a sensor response.

Because $\Sigmag$ is symmetric positive semidefinite,
\[
  \R^N
  =
  \Range(\Sigmag)
  \oplus
  \Null(\Sigmag),
\]
whereas the fundamental subspace decomposition of the lead field gives
\[
  \R^N
  =
  \Range(\mathbf L^{\mathsf T})
  \oplus
  \Null(\mathbf L).
\]
These two decompositions give rise to four canonical intersections,
\[
  \Range(\Sigmag)\cap\Null(\mathbf L),
  \qquad
  \Range(\Sigmag)\cap\Range(\mathbf L^{\mathsf T}),
\]
\[
  \Null(\Sigmag)\cap\Null(\mathbf L),
  \qquad
  \Null(\Sigmag)\cap\Range(\mathbf L^{\mathsf T}),
\]
which distinguish prior-supported from prior-unsupported directions and
likelihood-null from likelihood-visible directions. The two properties are
independent: a likelihood-null direction may either be supported or excluded
by the prior, and a prior-supported direction may be either visible or null to
the lead field.

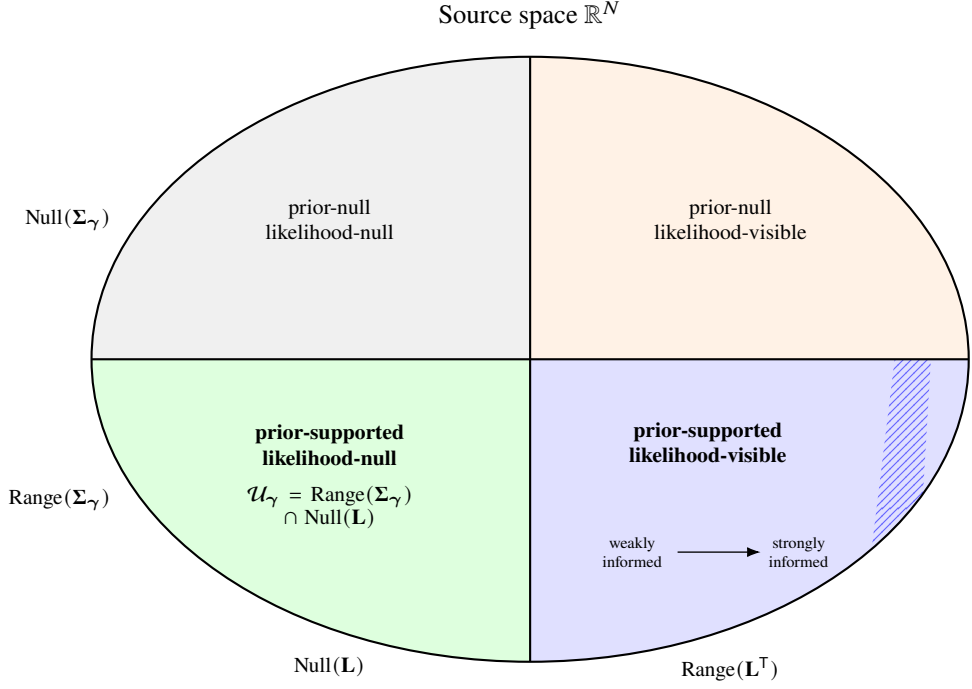
\begin{figure}[H]
\centering
\begin{tikzpicture}[
    x=1.00cm,
    y=1.00cm,
    every node/.style={
        font=\scriptsize,
        align=center,
        inner sep=1pt
    }
]

\def\rx{5.8}
\def\ry{4.0}

\begin{scope}
  \clip (0,0) ellipse (\rx cm and \ry cm);

  \fill[gray!12]
    (-\rx,0) rectangle (0,\ry);

  \fill[orange!10]
    (0,0) rectangle (\rx,\ry);

  \fill[green!13]
    (-\rx,-\ry) rectangle (0,0);

  \fill[blue!12]
    (0,-\ry) rectangle (\rx,0);

  \fill[
    pattern=north east lines,
    pattern color=blue!60
  ]
    (4.45,-3.05)
    -- (5.15,-2.55)
    -- (5.30,0)
    -- (4.82,0)
    -- cycle;
\end{scope}

\draw[thick] (0,0) ellipse (\rx cm and \ry cm);
\draw[thick] (-\rx,0) -- (\rx,0);
\draw[thick] (0,-\ry) -- (0,\ry);

\node[font=\small] at (0,4.55)
  {Source space $\R^N$};

\node[
  anchor=east,
  font=\scriptsize
] at (-5.50,1.85)
  {$\Null(\Sigmag)$};

\node[
  anchor=east,
  font=\scriptsize
] at (-5.50,-1.85)
  {$\Range(\Sigmag)$};

\node[
  anchor=north,
  font=\scriptsize
] at (-2.65,-3.90)
  {$\Null(\mathbf L)$};

\node[
  anchor=north,
  font=\scriptsize
] at (2.65,-3.90)
  {$\Range(\mathbf L^{\mathsf T})$};

\node[
  text width=2.45cm
] at (-2.65,1.85)
{
  prior-null\\
  likelihood-null
};

\node[
  text width=2.45cm
] at (2.65,1.85)
{
  prior-null\\
  likelihood-visible
};

\node[
  text width=2.75cm
] at (-2.65,-1.55)
{
  \textbf{prior-supported}\\
  \textbf{likelihood-null}\\[1.5mm]
  $\mathcal U_{\boldsymbol\gamma}
   =\Range(\Sigmag)$\\[-0.3mm]
  $\cap\,\Null(\mathbf L)$
};

\node[
  text width=2.65cm
] at (2.35,-1.10)
{
  \textbf{prior-supported}\\
  \textbf{likelihood-visible}
};

\node[
  font=\tiny,
  text width=1.05cm
] at (1.35,-2.55)
{
  weakly\\
  informed
};

\draw[
  -{Latex[length=1.8mm]},
  line width=0.5pt
]
  (1.95,-2.55) -- (3.05,-2.55);

\node[
  font=\tiny,
  text width=1.00cm
] at (3.55,-2.55)
{
  strongly\\
  informed
};

\end{tikzpicture}

\caption{
Schematic relation between prior support and likelihood visibility in physical
source space. Prior support is determined by $\Sigmag$, whereas sensing
visibility is determined by $\mathbf L$. The prior-supported likelihood-null
subspace is
$\mathcal U_{\boldsymbol\gamma}
=\Range(\Sigmag)\cap\Null(\mathbf L)$.
Within the prior-supported likelihood-visible sector, measurement
informativeness varies continuously from weakly to strongly informed
directions. The four sectors provide a schematic geometric interpretation and
need not form a direct-sum decomposition of $\R^N$.
}
\label{fig:prior-likelihood-geometry}
\end{figure}

The lower-left sector in \cref{fig:prior-likelihood-geometry} is of particular
interest for conditional uncertainty quantification. It consists of source
perturbations that remain admissible under the prior but are invisible to the
measurement operator. The standardized likelihood-null space derived above
maps exactly onto this physical subspace.

\begin{theorem}[Prior-supported component of the likelihood-null source space]
\label{thm:physical-null}

For fixed $\boldsymbol\gamma$, define
\begin{equation}
  \mathcal U_{\boldsymbol\gamma}
  :=
  \Range(\Sigmag)\cap\Null(\mathbf L).
  \label{eq:physical-nullspace}
\end{equation}
Then
\begin{equation}
  \mathcal U_{\boldsymbol\gamma}
  =
  \Bgamma\Null(\Agl)
  =
  \Range(\Bgamma\mathbf V_0).
  \label{eq:null-equivalence}
\end{equation}
Hence, the image of the standardized \emph{likelihood-null} subspace in
physical source coordinates is precisely the prior-supported component of the
lead-field nullspace.
\end{theorem}

\begin{proof}
First, let $\mathbf v\in\Null(\Agl)$. By definition,
\[
  \Agl\mathbf v
  =
  \Sigmal^{-1/2}\mathbf L\Bgamma\mathbf v
  =
  \mathbf0.
\]
Since $\Sigmal^{-1/2}$ is invertible,
\[
  \mathbf L\Bgamma\mathbf v=\mathbf0.
\]
Moreover,
\[
  \Bgamma\mathbf v
  \in
  \Range(\Bgamma)
  =
  \Range(\Sigmag).
\]
Therefore,
\[
  \Bgamma\mathbf v
  \in
  \Range(\Sigmag)\cap\Null(\mathbf L)
  =
  \mathcal U_{\boldsymbol\gamma},
\]
and hence
\[
  \Bgamma\Null(\Agl)
  \subseteq
  \mathcal U_{\boldsymbol\gamma}.
\]

Conversely, let
$\mathbf h\in\mathcal U_{\boldsymbol\gamma}$. Since
\[
  \mathbf h
  \in
  \Range(\Sigmag)
  =
  \Range(\Bgamma),
\]
there exists $\mathbf v\in\R^q$ such that
\[
  \mathbf h=\Bgamma\mathbf v.
\]
Because $\mathbf h\in\Null(\mathbf L)$,
\[
  \mathbf L\mathbf h=\mathbf0,
\]
and therefore
\[
  \Agl\mathbf v
  =
  \Sigmal^{-1/2}\mathbf L\Bgamma\mathbf v
  =
  \Sigmal^{-1/2}\mathbf L\mathbf h
  =
  \mathbf0.
\]
Thus $\mathbf v\in\Null(\Agl)$, implying
\[
  \mathcal U_{\boldsymbol\gamma}
  \subseteq
  \Bgamma\Null(\Agl).
\]
Combining the two inclusions yields
\[
  \mathcal U_{\boldsymbol\gamma}
  =
  \Bgamma\Null(\Agl).
\]

Finally, since the columns of $\mathbf V_0$ form an orthonormal basis of
$\Null(\Agl)$,
\[
  \Null(\Agl)=\Range(\mathbf V_0),
\]
and hence
\[
  \Bgamma\Null(\Agl)
  =
  \Range(\Bgamma\mathbf V_0).
\]
This proves \cref{eq:null-equivalence}.
\end{proof}

\begin{remark}
Prior support and likelihood visibility are independent properties. A fitted
zero variance $\widehat\gamma_i=0$ excludes the corresponding source
coordinate from the fitted prior-supported source space, irrespective of
whether that direction is visible or null to the lead field. Conversely, a
likelihood-null perturbation $\mathbf h\in\Null(\mathbf L)$ may lie either
inside or outside $\Range(\widehat{\Sigmag})$. Only the intersection
\[
  \Range(\widehat{\Sigmag})\cap\Null(\mathbf L)
\]
contributes likelihood-null uncertainty to the conditional posterior, whereas
directions in $\Null(\widehat{\Sigmag})$ carry no conditional posterior
variance. 
\end{remark}

\subsection{Spectral markers}

The singular values of the noise-whitened, prior-standardized operator quantify
how strongly the prior-supported spectral directions are represented in the
measurements. For the $i$th likelihood-visible right singular vector
$\mathbf v_i$, $i=1,\ldots,r$,
\[
  \Agl\mathbf v_i
  =
  s_i\mathbf u_i,
\]
where $\mathbf u_i$ is the corresponding left singular vector. Since $\mathbf v_i$ has unit norm in the standardized source coordinates, $s_i>0$ measures the magnitude of its noise-whitened sensor response.

By \cref{eq:pphessian}, the same direction satisfies
\[
  \mathbf H_{\boldsymbol\gamma,\boldsymbol\lambda}\mathbf v_i
  =
  s_i^2\mathbf v_i.
\]
We therefore define the \emph{spectral information value}
\[
  \kappa_i:=s_i^2.
\]
Its interpretation follows directly from the prior standardization. Since
\[
  \mathbf z\sim\mathcal N(\mathbf 0,\I_q),
\]
the prior precision along every unit spectral direction is one, whereas the
likelihood contributes
\[
  \mathbf v_i^{\mathsf T}
  \Agl^{\mathsf T}\Agl
  \mathbf v_i
  =
  s_i^2
  =
  \kappa_i.
\]
Hence, the \emph{spectral information value} $\kappa_i$ measures the likelihood precision relative to the unit prior precision along the $i$th standardized spectral direction. The corresponding posterior precision is therefore
\[
  \underbrace{1}_{\text{prior precision}}
  +
  \underbrace{\kappa_i}_{\text{likelihood precision}}
  =
  \underbrace{1+\kappa_i}_{\text{posterior precision}}.
\]
In the likelihood-visible coordinates, the posterior covariance derived above
gives
\[
  \Var(\alpha_i\mid\widetilde{\mathbf y})
  =
  \frac{1}{1+\kappa_i}.
\]
Since $\Var(\alpha_i)=1$ under the standardized prior, we define the
\emph{variance-retention fraction}
\begin{equation}
  \tau_i
  :=
  \frac{
    \Var(\alpha_i\mid\widetilde{\mathbf y})
  }{
    \Var(\alpha_i)
  }
  =
  \frac{1}{1+\kappa_i}
  =
  \frac{1}{1+s_i^2}.
  \label{eq:variance-ratio}
\end{equation}
Thus, the \emph{variance-retention fraction} $\tau_i$ gives the fraction of prior variance that remains after conditioning along the $i$th standardized spectral direction.

The complementary \emph{variance-reduction fraction} is
\begin{equation}
  \rho_i
  :=
  1-\tau_i
  =
  \frac{\kappa_i}{1+\kappa_i}
  =
  \frac{s_i^2}{1+s_i^2},
  \label{eq:reduction-fraction}
\end{equation}
so that
\[
  \tau_i+\rho_i=1.
\]
The \emph{variance-reduction fraction} $\rho_i$ therefore measures the fraction of prior variance removed by conditioning along the same spectral direction.


The three markers are equivalent representations of the same spectral information. We use $\kappa_i$ to express the likelihood precision relative to the standardized prior precision, and $\tau_i$ and $\rho_i$ to express the fractions of prior variance retained and removed by conditioning, respectively. In particular,
\[
  \kappa_i\gg1
  \quad\Longrightarrow\quad
  \tau_i\approx0,
  \qquad
  \rho_i\approx1,
\]
corresponding to a strongly data-informed direction, whereas
\[
  0<\kappa_i\ll1
  \quad\Longrightarrow\quad
  \tau_i\approx1,
  \qquad
  \rho_i\approx0,
\]
corresponding to a weakly informed but still likelihood-visible direction. For an exactly likelihood-null direction,
\[
  \kappa_i=0,
  \qquad
  \tau_i=1,
  \qquad
  \rho_i=0,
\]
so that conditioning leaves its prior variance unchanged.

In the homoscedastic case $\Sigmal=\sigma^2\I_M$, let $\bar s_i$ denote the corresponding singular value of $\mathbf L\Bgamma$. Since
\[
  \Agl
  =
  \sigma^{-1}\mathbf L\Bgamma,
\]
we have
\[
  s_i
  =
  \frac{\bar s_i}{\sigma},
  \qquad
  \kappa_i
  =
  \frac{\bar s_i^2}{\sigma^2}.
\]
Thus, the \emph{spectral information value} combines the lead-field/prior geometry with the measurement-noise scale.

For the $i$th likelihood-visible spectral direction, define its physical
source-space representation by
\[
  \boldsymbol\psi_i
  :=
  \Bgamma\mathbf v_i
  \in\R^N.
\]
Its rank-one prior covariance contribution satisfies
\begin{equation}
  \underbrace{
    \boldsymbol\psi_i\boldsymbol\psi_i^{\mathsf T}
  }_{\text{prior contribution}}
  =
  \underbrace{
    \tau_i
    \boldsymbol\psi_i\boldsymbol\psi_i^{\mathsf T}
  }_{\text{posterior contribution}}
  +
  \underbrace{
    \rho_i
    \boldsymbol\psi_i\boldsymbol\psi_i^{\mathsf T}
  }_{\text{likelihood-driven reduction}}.
  \label{eq:modewise-physical}
\end{equation}
Hence, the \emph{variance-retention fraction} and \emph{variance-reduction fraction} provide the normalized, direction-wise counterpart of the prior-to-posterior covariance accounting derived above.

Because a spectral direction generally combines several source coordinates,
$\boldsymbol\psi_i=\Bgamma\mathbf v_i$ need not coincide with a canonical
source direction. Its covariance contribution therefore cannot, in general,
be reduced to a scalar expression such as $\gamma_i\tau_i$.

The interpretation of the three spectral markers is summarized in
\cref{tab:regimes}.

\begin{table}[H]
\centering
\caption{
Interpretation of the \emph{spectral information value} $\kappa$,
\emph{variance-retention fraction} $\tau$, and
\emph{variance-reduction fraction} $\rho$.
}
\label{tab:regimes}
\begin{tabular}{@{}>{\raggedright\arraybackslash}p{37mm}
                    >{\centering\arraybackslash}p{20mm}
                    >{\centering\arraybackslash}p{20mm}
                    >{\raggedright\arraybackslash}p{57mm}@{}}
\toprule
Spectral regime
& Retention fraction $\tau$
& Reduction fraction $\rho$
& Interpretation \\
\midrule

Strongly data-informed,
$\kappa\gg1$
&
$\tau\approx0$
&
$\rho\approx1$
&
The measurements remove most of the prior variance along the direction.\\

Balanced information,
$\kappa=1$
&
$\tau=1/2$
&
$\rho=1/2$
&
Prior and likelihood precision are equal, so half of the prior variance
remains and half is removed.\\

Weakly data-informed but visible,
$0<\kappa\ll1$
&
$\tau\approx1$
&
$\rho\approx0$
&
The likelihood contributes only weak information, so most of the prior
variance remains after conditioning.\\

Exactly likelihood-null,
$\kappa=0$
&
$\tau=1$
&
$\rho=0$
&
The measurements provide no information along the direction, and its prior variance is preserved in the posterior.\\

\bottomrule
\end{tabular}
\end{table}

\section{Numerical results}
\label{sec:numerical}

\subsection{Simulation setup}
\label{sec:simulation-setup}

We use a subject-specific MEG forward model derived from the CamCAN resource and processed with MNE-Python \citep{TaylorEtAl2017,GramfortEtAl2014}. The forward solution was constructed from an individual T1-weighted MRI using a boundary-element head model, with a minimum source-to-inner-skull distance of $3$ mm. The original forward model comprises $204$ MEG gradiometer channels and $4686$ candidate cortical source locations. For the present numerical study, we use a precomputed fixed-orientation subset,
\(\mathbf L\in\mathbb R^{204\times1284},\)
while preserving the same source orientations in both simulation and inverse reconstruction. Retaining all sensors while reducing the source
grid makes the Type-II fitting and complete spectral covariance analysis computationally tractable without altering the strongly underdetermined M/EEG setting.

We simulate $T=10$ time samples with $\texttt{nnz}=10$ active source
locations selected uniformly without replacement. The corresponding source
amplitudes are drawn independently from a standard Gaussian distribution,
yielding a sparse source matrix
$\mathbf X_{\mathrm{true}}\in\mathbb R^{1284\times10}$ and noise-free
sensor measurements
\[
    \mathbf Y_{\mathrm{signal}}
    =
    \mathbf L\mathbf X_{\mathrm{true}}.
\]

Sensor noise is modeled as zero-mean Gaussian with a random diagonal
covariance. A positive-semidefinite matrix is drawn from a Wishart
distribution with $M$ degrees of freedom and scale
$10^{-3}\mathbf I_M$, and its diagonal is retained to define
$\boldsymbol\Sigma_{\mathrm{noise}}$. The simulated noise
$\mathbf Y_{\mathrm{noise}}$ is then scaled relative to the signal energy,
and the final measurements are generated as
\begin{equation}
    \mathbf Y
    =
    \mathbf Y_{\mathrm{signal}}
    +
    \left(
        \frac{1-\alpha_{\mathrm{SNR}}}{\alpha_{\mathrm{SNR}}}
        \frac{
            \|\mathbf Y_{\mathrm{signal}}\|_{\mathrm F}
        }{
            \|\mathbf Y_{\mathrm{noise}}\|_{\mathrm F}
        }
    \right)
    \mathbf Y_{\mathrm{noise}},
    \label{eq:simulation-model}
\end{equation}
where $\alpha_{\mathrm{SNR}}$ controls the signal-to-noise level. We set
$\alpha_{\mathrm{SNR}}=0.99$, corresponding to a Frobenius-norm SNR of approximately $39.91\,\mathrm{dB}$. Random seed $42$ is used throughout, and the same simulated realization is provided to both inverse solvers.

We compare two adaptive Type-II Bayesian inverse solvers. The first is Bayesian minimum norm (BMN), which learns a single source-variance scale and a homoscedastic noise variance
\citep{WipfEtAl2007,SekiharaNagarajan2015}. To reduce sensitivity to spatial variations in lead-field norm, BMN is applied in sLORETA-normalized coordinates \citep{PascualMarqui2002}. Let
$\mathbf W=\diag(w_1,\ldots,w_N)$ denote the fixed normalization and write
\[
    \mathbf x(t)=\mathbf W\mathbf u(t),
    \qquad
    \mathbf u(t)\mid\gamma
    \sim
    \mathcal N(\mathbf0,\gamma\mathbf I_N).
\]
After undoing the solver's internal numerical scaling, the physical source and noise covariances used in the spectral analysis are
\[
    \boldsymbol\Sigma_{\gamma}^{\mathrm{BMN}}
    =
    \gamma_{\mathrm{phys}}
    \mathbf W\mathbf W^{\mathsf T}
    =
    \diag(
        \gamma_{\mathrm{phys}}w_1^2,\ldots,
        \gamma_{\mathrm{phys}}w_N^2
    ),
    \qquad
    \boldsymbol\Sigma_{\lambda}^{\mathrm{BMN}}
    =
    \lambda\mathbf I_M.
\]
Thus, BMN shares a single variance scale in the normalized coordinates and retains the full source support, while its physical source covariance is generally depth-weighted rather than proportional to $\mathbf I_N$.

The second solver is $\gamma$-MAP (Champagne), a sparse Bayesian learning method that assigns a separate variance hyperparameter to each source \citep{WipfNagarajan2009,WipfEtAl2010}. Its source covariance is
\[
    \boldsymbol\Sigma_{\gamma}
    =
    \diag(\gamma_1,\ldots,\gamma_N),
\]
and adaptive noise learning is performed with a sensor-wise diagonal noise covariance
\citep{HashemiEtAl2021,CaiEtAl2021},
\[
    \boldsymbol\Sigma_{\lambda}
    =
    \diag(\lambda_1,\ldots,\lambda_M).
\]
For both solvers, the Type-II marginal-likelihood objective is nonconvex in the covariance hyperparameters. We therefore employ convex-bounding majorization-minimization update rules to construct tractable surrogate problems and iteratively update the source and noise hyperparameters
\citep{HashemiEtAl2021,CaiEtAl2021}.

For $\gamma$-MAP, very small positive source variances may remain after the Type-II iterations instead of becoming exactly zero. We therefore apply the relative threshold
\[
    \widehat{\mathcal I}_{\gamma}
    :=
    \left\{
        j:
        \widehat{\gamma}_j
        \geq
        10^{-3}
        \max_k \widehat{\gamma}_k
    \right\}.
\]
The model is subsequently refitted on the restricted lead field
$\mathbf L_{\widehat{\mathcal I}_{\gamma}}$, and the resulting hyperparameters are used for the spectral analysis. Hence, the prior-supported dimension is
\[
    q=|\widehat{\mathcal I}_{\gamma}|.
\]
Let $\mathbf P_{\widehat{\mathcal I}_{\gamma}}
\in\mathbb R^{N\times q}$ denote the canonical selection matrix whose columns are the standard basis vectors associated with the retained source indices. The fitted prior factor is then
\[
    \mathbf B_{\widehat{\boldsymbol\gamma}}
    =
    \mathbf P_{\widehat{\mathcal I}_{\gamma}}
    \diag\!\left(
        \sqrt{\widehat{\gamma}_j}:
        j\in\widehat{\mathcal I}_{\gamma}
    \right).
\]
BMN with sLORETA normalization does not perform source pruning and therefore retains the complete source space, so that $q=N$.

After hyperparameter learning, the fitted source and noise covariance parameters are held fixed. For each method, the spectral analysis is based on the corresponding noise-whitened, prior-standardized operator
\[
    \mathbf A_{\widehat{\boldsymbol\gamma},
               \widehat{\boldsymbol\lambda}}
    =
    \widehat{\boldsymbol\Sigma}_{\lambda}^{-1/2}
    \mathbf L
    \mathbf B_{\widehat{\boldsymbol\gamma}},
\]
where $\mathbf B_{\widehat{\boldsymbol\gamma}}$ denotes the fitted prior factor associated with the corresponding solver. Its singular spectrum
is then used to evaluate $s_i$, $\kappa_i$, $\tau_i$, and $\rho_i$.

\subsection{Results}
\label{sec:numerical-results}

Table~\ref{tab:spectral-results} summarizes the fitted
prior-supported dimensions and the corresponding spectral structure of
the two Type-II models.

\begin{table}[H]
\centering
\begin{tabular}{@{}lrrrrr@{}}
\toprule
Method
& $q$
& $r$
& $q-r$
& $\kappa_i>1$
& $0<\kappa_i<1$ \\
\midrule
BMN with sLORETA
& 1284
& 204
& 1080
& 140
& 64
\\
$\gamma$-MAP
& 11
& 11
& 0
& 10
& 1
\\
\bottomrule
\end{tabular}

\caption{Spectral structure of the fitted models.}
\label{tab:spectral-results}
\end{table}

BMN with sLORETA normalization retains the complete
$1284$-dimensional prior-supported source space. The corresponding
noise-whitened, prior-standardized operator has rank $204$, leaving
$1080$ exact likelihood-null directions. Among the $204$
likelihood-visible directions, $140$ satisfy $\kappa_i>1$, so that the
likelihood precision exceeds the standardized prior precision, whereas
$64$ satisfy $0<\kappa_i<1$ and remain weakly informed. The information values span approximately
\[
    1.20\times10^{-4}
    \leq
    \kappa_i
    \leq
    8.31\times10^{5}.
\]

The fitted $\gamma$-MAP model yields a markedly different geometry.
After thresholding and refitting, only $11$ source locations remain in
the prior-supported space, of which $10$ coincide with the simulated
generators. The restricted operator has numerical rank $r=q=11$, so no exact likelihood-null direction remains within the fitted support. 
Ten directions satisfy $\kappa_i>1$, whereas the
remaining direction has
\[
    \kappa_{11}=0.656,
    \qquad
    \tau_{11}=0.604,
    \qquad
    \rho_{11}=0.396.
\]
Thus, this direction is still likelihood-visible, but only $39.6\%$ of its prior variance is removed by the measurements and $60.4\%$ remains
in the posterior. By contrast, the other ten directions have very large information values, with
\[
    \kappa_{\max}\approx1.25\times10^{6},
\]
and consequently $\rho_i$ is close to one.

The difference between the two fitted models is also visible directly in source space. Figure~\ref{fig:source-reconstruction-final} compares the simulated source pattern with the corresponding posterior-mean reconstructions. BMN with sLORETA normalization produces a dense and spatially diffuse estimate, whereas $\gamma$-MAP concentrates the reconstructed activity almost entirely on the simulated sparse support. In particular, the ten largest $\gamma$-MAP posterior-mean amplitudes coincide with the ten simulated source locations, while two of the ten simulated generators appear among the ten largest BMN with sLORETA amplitudes. \rev{To connect the reconstruction with the uncertainty accounting, the lower part of the figure additionally reports the trace shares of the exact likelihood-visible and likelihood-null posterior covariance blocks. The bars are normalized by the posterior trace of each fitted model, while the adjacent annotations report the corresponding posterior-to-prior trace ratio; this separates the composition of the remaining uncertainty from its absolute contraction relative to the fitted prior.}

\begin{figure}[!htbp]
    \centering
    \includegraphics[width=0.98\linewidth]
    {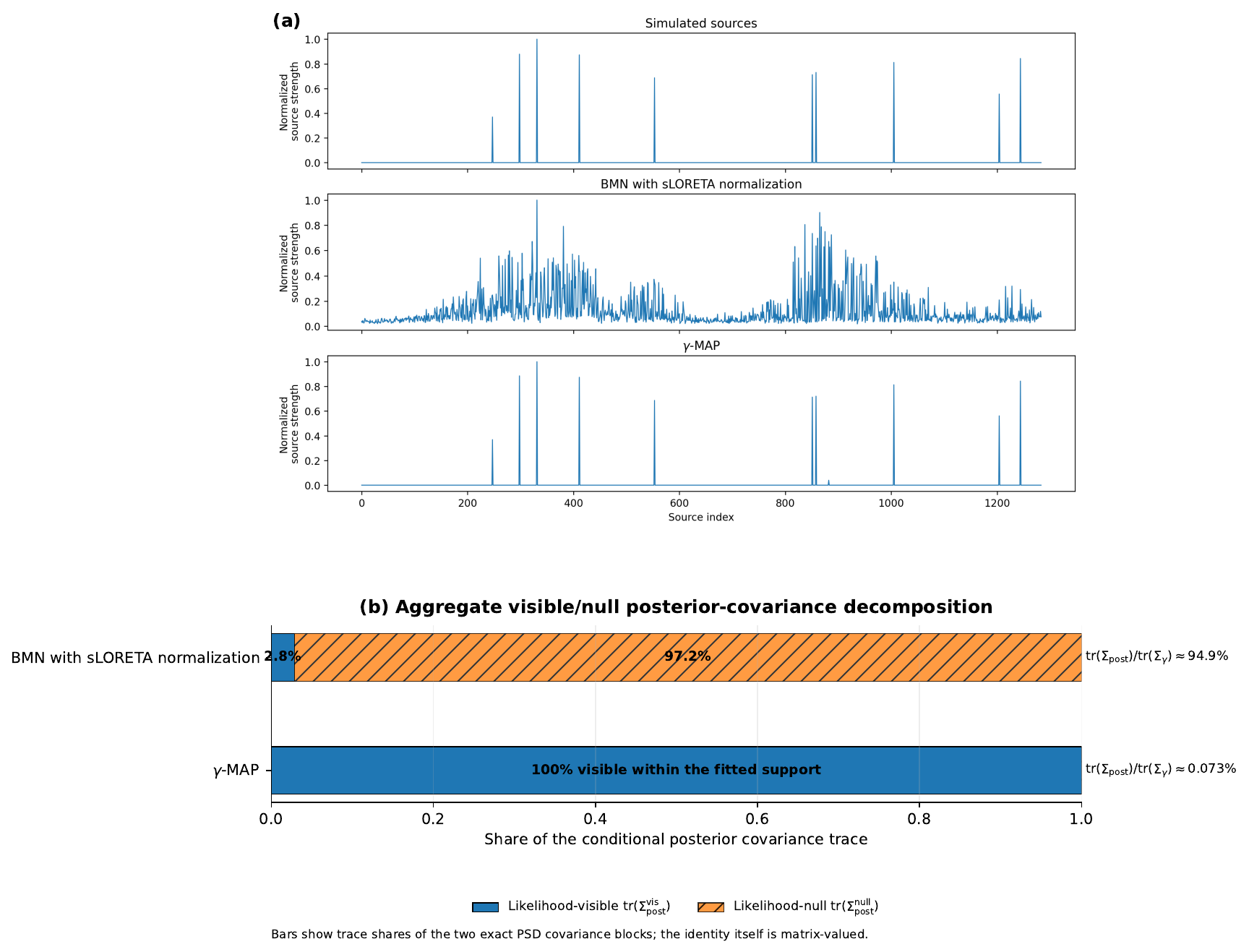}
    \caption{
    Normalized source reconstruction and aggregate posterior-covariance
    decomposition for the simulated example. Panel (a) shows the simulated
    sources, the posterior-mean reconstruction obtained with BMN with
    sLORETA normalization, and the reconstruction obtained with
    $\gamma$-MAP. \rev{Panel (b) summarizes the exact additive covariance
    identity at trace level. For BMN, approximately $2.8\%$ of the
    conditional posterior trace is carried by the likelihood-visible block
    and $97.2\%$ by the likelihood-null block; the total posterior trace is
    approximately $94.9\%$ of the fitted prior trace. For $\gamma$-MAP,
    the likelihood-null block is zero within the pruned fitted support and
    the remaining conditional posterior covariance is entirely
    likelihood-visible, while its total trace is approximately $0.073\%$
    of the fitted prior trace. The latter statement is conditional on the
    learned support and does not imply that the original full source-space
    lead field is injective.}
    }
    \label{fig:source-reconstruction-final}
\end{figure}

\rev{The trace summary quantifies the aggregate visible/null split, but it does not identify how strongly individual retained spectral directions are constrained by the measurements.} To relate the spectral markers to the known simulation support, we evaluate
the fraction of each physical spectral mode
$\boldsymbol\psi_i=\mathbf B_{\widehat{\boldsymbol\gamma}}\mathbf v_i$
concentrated on the true sources,
\[
    \eta_i
    =
    \frac{
        \sum_{j\in\mathcal S_{\rm true}}\psi_{ji}^{2}
    }{
        \|\boldsymbol\psi_i\|_2^2
    }.
\]
For the ten strongly informed $\gamma$-MAP directions,
$\eta_i>0.99998$, whereas the single weakly informed direction has
$\eta_{11}=0.392$ and is dominated by the additional retained source
outside the simulated support. Hence, the spectral markers distinguish
between directions that are merely retained by the Type-II fit and
those that are strongly supported by the measurements.

\begin{figure}[!htbp!htbp]
    \centering
    \includegraphics[width=0.88\linewidth]{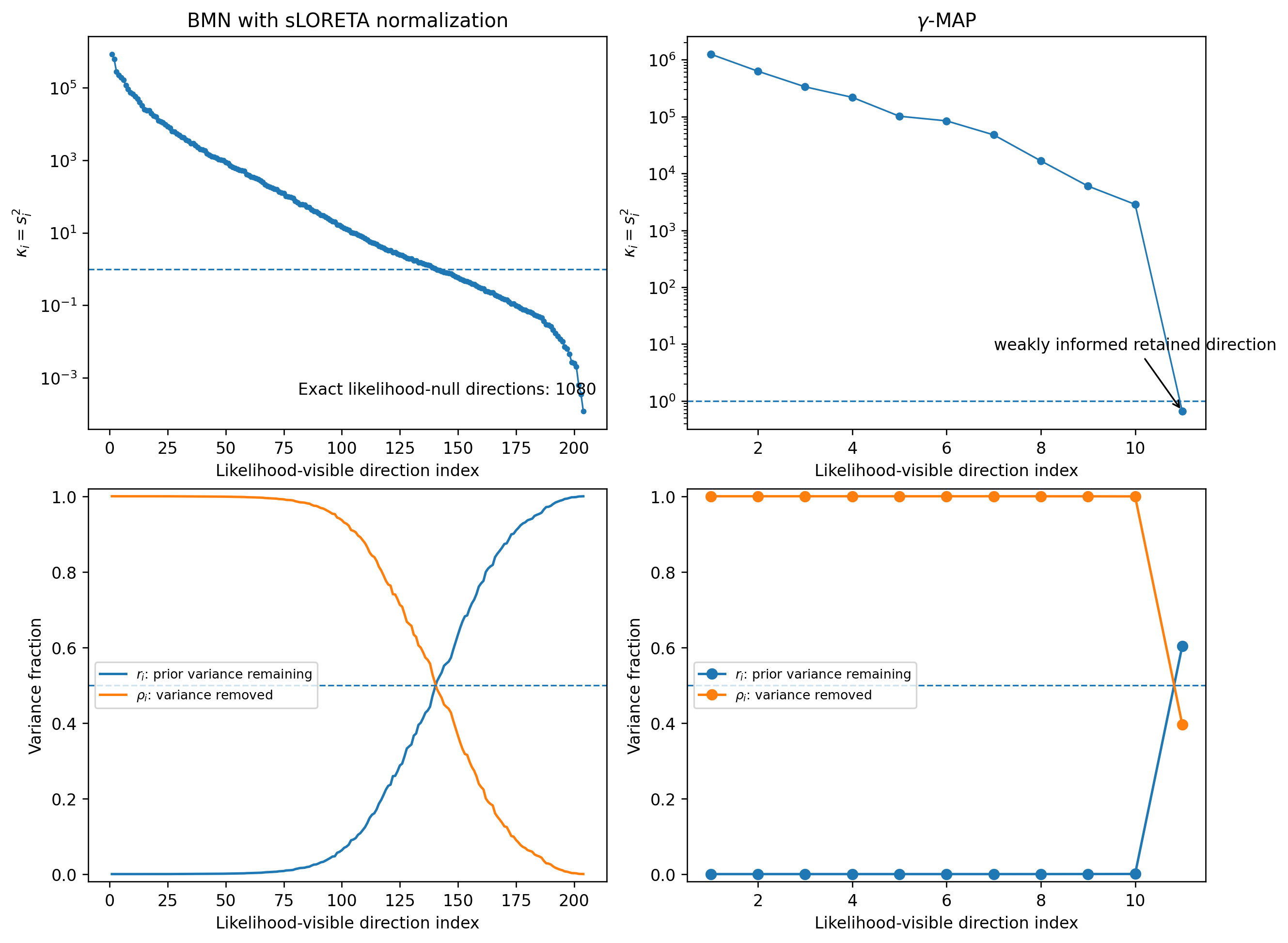}
    \caption{
    Spectral information and variance markers for BMN with sLORETA
    normalization and $\gamma$-MAP. The upper panels show
    $\kappa_i=s_i^2$ with the reference level $\kappa_i=1$.
    The lower panels show the corresponding fractions of prior variance
    remaining in the posterior, $\tau_i$, and removed by the data,
    $\rho_i$. Exact likelihood-null directions of BMN with sLORETA are
    reported separately and are not included among the positive
    singular values.
    }
    \label{fig:spectral-markers-final}
\end{figure}

The effect of the spectral structure on the posterior covariance can
also be quantified through the additive and subtractive identities.
At each source coordinate $j$,
\[
    (\boldsymbol\Sigma_{\rm post})_{jj}
    =
    (\boldsymbol\Sigma_{\rm post}^{\rm vis})_{jj}
    +
    (\boldsymbol\Sigma_{\rm post}^{\rm null})_{jj},
\]
and equivalently,
\[
    (\boldsymbol\Sigma_{\rm post})_{jj}
    =
    (\boldsymbol\Sigma_{\gamma})_{jj}
    -
    (\mathbf R_{\rm data})_{jj}.
\]

For BMN with sLORETA normalization, the covariance component associated
with the $1080$-dimensional exact likelihood-null subspace accounts for
approximately $97.2\%$ of the total conditional posterior trace.
This component remains unchanged from prior to posterior, and its
trace equals approximately $92.2\%$ of the fitted prior trace.
Conditioning removes approximately $5.1\%$ of the fitted prior trace,
leaving a total posterior trace equal to approximately $94.9\%$ of
the fitted prior trace.

For $\gamma$-MAP, the fitted support contains no exact
likelihood-null direction. Approximately $99.93\%$ of the fitted prior
trace is removed by conditioning, leaving about $0.073\%$ in the
posterior. The source-wise decomposition is particularly informative
for the retained support: for each of the ten simulated sources, at
least $99.68\%$ of the fitted prior variance is removed, whereas the
additional retained source keeps approximately $60.1\%$ of its prior
variance. This behavior agrees with the weak spectral marker
$\kappa_{11}=0.656$ and shows that the spectral and covariance
decompositions identify the same poorly constrained component from two
complementary perspectives.

The adaptive noise estimates are also close to the simulated noise
level. The mean true sensor-noise variance is
$3.10\times10^{-12}$. BMN with sLORETA normalization estimates the
common noise variance as $3.27\times10^{-12}$, whereas $\gamma$-MAP
estimates a sensor-wise diagonal noise covariance with mean variance
$2.99\times10^{-12}$.

Finally, the covariance identities are recovered numerically to high
precision. The relative residuals are approximately
$4.8\times10^{-13}$ for BMN with sLORETA normalization and
$1.5\times10^{-8}$ for $\gamma$-MAP, confirming consistency between
the directly computed conditional posterior covariance and its
spectral additive and prior-minus-reduction representations.

\section{Discussion and conclusion}
\label{sec:discussion}

The numerical results show that the geometry of the fitted prior support has a direct effect on how posterior uncertainty should be interpreted.
For BMN with sLORETA normalization, the full source space remains prior-supported, and the large likelihood-null subspace carries most of
the remaining uncertainty. In fact, the null component accounts for approximately $92.2\%$ of the fitted prior trace, although the nullspace
dimension itself represents about $84.1\%$ of the prior-supported coordinates. 


The $\gamma$-MAP result illustrates the complementary effect of Type-II
support adaptation. After pruning and refitting, the fitted prior support
becomes small enough that no exact likelihood-null direction remains
within the conditional model. Source directions outside the fitted support
have been assigned zero prior variance by the empirical-Bayes estimate
before the spectral decomposition is formed.

Importantly, this model-induced removal of variance does not imply that the
corresponding physical ambiguity has disappeared. In strongly ill-conditioned
source configurations, two nearby candidate sources may have highly similar
lead fields, yet Type-II optimization may retain one while assigning zero
variance to the other. The resulting conditional posterior then reports no
uncertainty along the pruned coordinate, although uncertainty about the
support selection itself remains unresolved. This is a direct consequence of
using plug-in hyperparameter estimates: the present covariance quantifies
uncertainty conditional on the fitted support and does not propagate
uncertainty over alternative hyperparameter or support configurations.

A related distinction concerns marginal and joint posterior degeneracy.
In our calibration study, $\gamma$-MAP is combined with sparse basis
field expansions (sFLEX) to mitigate zero-variance posterior marginals
at individual voxels \citep{HuseynovEtAlCalibration}. The spatially
extended basis functions can distribute uncertainty from retained
coefficients across voxels, but they do not restore variability outside
the span of the retained basis functions. With $q$ retained coefficients,
the source-space posterior covariance has rank at most $q$, even when
all voxelwise marginal variances are positive. Consequently, at least
$N-q$ source-space directions have zero posterior variance; when
$q\leq M$, there are at least $N-M$ such directions. Physically
admissible source configurations outside the retained span, including
spatially extended patterns, remain excluded by the fitted model.
Thus, resolving voxelwise marginal degeneracy does not resolve
degeneracy of the joint posterior.

The remaining weak spectral direction provides the complementary case: a
source direction can remain within the fitted prior support while still being
only weakly constrained by the measurements. Thus, prior-support selection,
likelihood informativeness, and uncertainty about the support-selection step
represent distinct aspects of the inverse problem and should be interpreted
separately.


The spectral markers are therefore useful beyond posterior-mean source reconstruction. A small positive $\kappa_i$ identifies a direction that
is represented in the measurements but only weakly informed relative to its prior uncertainty, so that  improved measurement precision may
further reduce its posterior variance. In contrast, $\kappa_i=0$ identifies an exact prior-supported likelihood-null direction, for which
increasing measurement precision alone cannot reduce uncertainty as long as the sensing geometry is unchanged. The corresponding covariance
identities provide the same distinction at the level of posterior uncertainty: they separate uncertainty that remains along visible directions from uncertainty preserved in the mapped nullspace, while the prior-minus-reduction form quantifies the covariance actually removed by the data.

Several limitations delimit the numerical conclusions. The present benchmark is based on a single controlled high-SNR realization, one fixed-orientation lead field, a small number of time samples, and plug-in Type-II covariance estimates. Moreover, BMN and $\gamma$-MAP
use different noise-covariance parameterizations, so the experiment compares complete adaptive Type-II pipelines rather than isolating source sparsity as the only changing factor. The relative support
threshold used for $\gamma$-MAP is also part of the reported procedure. A broader evaluation should therefore include repeated realizations and
multiple SNR levels, matched noise models, sensitivity to the support threshold, and calibration studies on both simulated and real M/EEG data. A fully Bayesian treatment of the source and noise hyperparameters would additionally propagate hyperparameter uncertainty that is not represented by the present conditional posterior covariance. Uncertainty in the forward model is likewise not represented,
since the lead-field matrix is treated as fixed throughout the analysis.

In summary, the proposed framework turns the conditional Gaussian posterior covariance into an exact spectral accounting of what is informed by the measurements, what remains invisible under the fitted sensing geometry and prior support, and how much covariance is removed by conditioning. The BMN and $\gamma$-MAP examples show why these
quantities must be interpreted jointly with the fitted prior support. This distinction is practically relevant when determining whether
residual uncertainty can be reduced by improving measurement precision, requires a change in sensing geometry, or has instead been removed from
the conditional model through prior-support adaptation.

\bibliographystyle{plainnat}
\bibliography{references}

\end{document}